\documentclass{article}

\usepackage[final, nonatbib]{neurips_2020}

\usepackage[utf8]{inputenc} 
\usepackage[T1]{fontenc}    
\usepackage{hyperref}       
\usepackage{url}            
\usepackage{booktabs}       
\usepackage{amsfonts}       
\usepackage{nicefrac}       
\usepackage{microtype}      

\usepackage{microtype}
\usepackage{graphicx}
\usepackage{subfigure}
\usepackage{epsfig}
\usepackage{mathrsfs}
\usepackage{latexsym}
\usepackage{color}
\usepackage[usenames,dvipsnames]{xcolor}
\usepackage{amsmath,amssymb}
\usepackage{verbatim}
\definecolor{red}{rgb}{1,0,0}
\usepackage{cite}
\usepackage{graphicx}
\usepackage{xcolor,colortbl}
\usepackage{authblk}
\usepackage{float}
\usepackage{amsthm}
\usepackage{tabu}
\usepackage{array, boldline, makecell, booktabs}

\usepackage{enumitem}

\usepackage{tikz,tikz-3dplot}
\usetikzlibrary{positioning}

\newtheorem{theorem}{Theorem}

\def\bea{\begin{eqnarray}} 
\def\eea{\end{eqnarray}}
\def\be{\begin{equation}} 
\def\ee{\end{equation}} 
\def\ba{\begin{array}}
\def\ea{\end{array}}

\title{SuperCoder: Program Learning Under Noisy Conditions From Superposition of States}

\author{\textbf{Ali Davody}$^*$ \textsuperscript{\!\!1} \quad \textbf{Mahmoud Safari}\thanks{Equal contribution.}   \textsuperscript{\hspace{2.5pt}2} \quad \textbf{R\u azvan V.~Florian} \textsuperscript{2} \\
	\textsuperscript{1} 
	Spoken Language Systems Group, Saarland Informatics Campus, Saarland University, Germany\\
	\textsuperscript{2} Romanian Institute of Science and Technology, str.~Virgil Fulicea  nr. 3, 400022 Cluj-Napoca, Romania \\ \vspace{1mm}
	\tt\small adavody@lsv.uni-saarland.de \quad  \tt\small safari@rist.ro 
	\quad \tt\small florian@rist.ro
}

\begin{document}

\maketitle

\begin{abstract}
	We propose a new method of program learning in a Domain Specific Language (DSL) which is based on gradient descent with no direct search. The first component of our method is a probabilistic representation of the DSL variables. At each timestep in the program sequence, different DSL functions are applied on the DSL variables with a certain probability, leading to different possible outcomes. Rather than handling all these outputs separately, whose number grows exponentially with each timestep, we collect them into a superposition of variables which captures the information in a single, but fuzzy, state. This state is to be contrasted at the final timestep with the ground-truth output, through a loss function. The second component of our method is an attention-based recurrent neural network, which provides an appropriate initialization point for the gradient descent that optimizes the probabilistic representation. The method we have developed surpasses the state-of-the-art for synthesising long programs and is able to learn programs under noise.
\end{abstract}

\section{Introduction}
\label{s:intro}

Despite years of research, the ambitious goal of automatically generating a program consistent with a given specification is still one of the active fields in machine learning and artificial intelligence. In inductive program synthesis, a set of input/output (I/O) examples is used as a specification from which the program is to be inferred. Existing approaches include many search algorithms that have been developed to synthesize programs by using a predefined Domain Specific Language (DSL). Recent progress in deep learning has triggered a boost in this field, leading to methods based on neural networks.  Approaches for neural program learning include either ``program synthesis'', where a symbolic representation of the program is inferred \cite{zohar2018automatic, balog2016deepcoder, parisotto2016neurosymbolic, alex2016terpret, bonjak2016programming}, or ``program induction'' where a neural network predicts the output corresponding to a given input \cite{graves2014neural, graves2016hybrid, kurach2015neural, kaiser2015neural, NIPS2015_5857, reed2015neural, neelakantan2015neural}.
For a clear overview and comparisons between the two approaches see \cite{devlin2017robustfill}. Recently, search techniques have been combined with neural methods where the output of the neural network guides the search algorithm. The first example of this is \cite{balog2016deepcoder}, which shows that integrating neural networks with search-based methods can significantly improve the performance compared to search-based techniques or neural-based approaches in isolation. In a more recent work \cite{zohar2018automatic} along the same lines, the authors employ a step-wise approach and a beam search, and further train a second network to predict the variables to be discarded from the tree search. This improves the performance and extends the effectiveness of the approach to longer programs.

Here we develop a new way of combining a neural network with some sort of ``search'' method, where the search does not involve a brute-force verification of the predicted programs but is instead a gradient-descent-based optimization process. There are several specific features of our approach. First, we employ an alternative (compact) representation for the variables on which DSL rules act, which can be interpreted as the probability distribution over all possible variables. This representation can capture the original deterministic cases (sharp distributions) but is also capable of representing the superposition of several variables with their corresponding probabilities. These superpositions inspire the name of our approach, SuperCoder. Second, the DSL rules must be realized on this probabilistic state space as a differentiable map, such that they transform one probabilistic state to another. Starting from a single sharp input, the different DSL rules may act on it with different probabilities, and this gives rise to several different states, each of which can occur with a certain probability. However, as a result of this more convenient probabilistic representation, these different states can be summed up (superposed) to form a single, but fuzzy, state. This continues until we come to the final output which is again a single fuzzy variable. Finally, this last state is to be contrasted with the ground-truth output and optimized with respect to a suitable cost function. The optimization parameters are the probabilities with which the DSL rules are applied on the state at different steps of the program. The trained probabilities then point to a particular program which ``most likely'' transforms the input to the output. Unlike brute-force search methods, including the state-of-the-art method \cite{zohar2018automatic} where the number of variables still grows considerably with program length, the structure of this (fuzzy) state does not change with the number of program steps. The only thing that changes is the values of its entries. This is to some extent responsible for the better performance of our method on long programs compared to existing approaches. We have developed this method based on inspiration from quantum superposition in physics.

Ref. \cite{alex2016terpret} also uses a differentiable fuzzy system of distributions to represent variables for a program synthesis task, however while their primitives are limited to assignments and conditionals based on comparisons with constants, our primitives extend to a range of arithmetic operations operating on lists of integers and to other list operations such as \texttt{head} and \texttt{tail}. Also, \cite{kurach2015neural} uses variables that represent probability distributions, but the modules that perform operations on these variables do not operate on lists of variables like in our case but on pairs of inputs, and these inputs are not directly the variables but involve outputs of a neural network controller. The architectures of both \cite{alex2016terpret} and  \cite{kurach2015neural} are significantly different than ours.

Another specific of our method is combining the output of a neural program synthesizer with the previously mentioned gradient-descent-based optimization. The output of the neural network is used as a starting point for the additional optimization. We therefore use a novel pipeline of two modules that are using distinct approaches for generating and then refining a probability distribution over synthesised programs\footnote{Our implementation of this method is available at https://github.com/rist-ro/SuperCoder}.

Our main contributions are summarized as follows: 
\begin{itemize}[leftmargin=*]
	\item[$\bullet$] We have employed insights from quantum physics to develop a novel gradient-based search method for program learning by introducing a new continuous relaxation of the search space which combines a probabilistic interpretation of DSL variables and the idea of superposition of states. The gradient-based nature of the method makes it capable of handling noise in the data, on tasks where the goal is to predict the output.
	\item[$\bullet$] As an essential ingredient in our differentiable search method, we introduce a representation of the DSL variables that is highly compressed, and yet captures enough information to allow identifying a program consistent with a given input-output. This is demonstrated through a theorem we prove in Section 4.
	\item[$\bullet$] We integrate our differentiable search method with a neural network, which provides a suitable initial point for this gradient-based optimization. This significantly improves the performance. 
	\item[$\bullet$] We have tested our SuperCoder method on two different program learning tasks, including one with noisy data. For long programs, i.e. those with a sequence of at least 11 operations, our method is shown to outperform the state-of-the-art.
\end{itemize}

\section{Setup of the problem}

In order to demonstrate our approach and to compare it with a state-of-the-art method \cite{zohar2018automatic}, we adopt the same problem setup as in \cite{balog2016deepcoder, zohar2018automatic} where the I/O of the program can either be an integer or a list of integers. The integers (whether inside a list or as separate variables) range from a minimum value $\ell_\textrm{min}$ to a maximum value $\ell_\textrm{max}$. The lists have a maximum length of $L$. The DSL functions include simple operations on these variables,  such as applying an arithmetic operation on all list elements, taking the first or last element of the list, etc. Operations on the DSL variables may also result in a null variable. A simple example is increasing list elements by one when the maximum possible integer is already within the list, so that the result falls outside the allowed range. We refer to the three cases ``integer'', ``list'' and ``null'' as variable {\it type}. 

The goal is to infer the program (i.e. sequence of operations) given I/O pairs (synthesis). The synthesized programs can then be used to predict the outputs corresponding to new inputs.

The central idea in our approach is to avoid brute-force search algorithms by adopting a probabilistic point of view. A key step in the construction is to represent our variables (integers, list of integers or the null variable) with a 3-dimensional tensor $ \psi_{i,j,k}$ with 0 and 1 entries, having a shape $(L+2,L,d)$ where $L$ is the maximum length of lists and $ d = \ell_\textrm{max} - \ell_\textrm{min} +1 $ is the number of possible values of integers. Roughly, the first dimension determines the type of the variable: null, integer, or list of length $ 1 $ through $ L $. The second dimension specifies the position in the list, and the third dimension refers to the value of the integer at that position. When there is no confusion, we will drop the commas between indices ($\psi_{ijk} \equiv \psi_{i,j,k}$).

Specifically, the null state corresponds to a tensor whose only nonzero element is $ \psi_{000} =1 $. An integer variable $ k $ corresponds to $ \psi_{1, 0, k-\ell_\textrm{min}} =1 $ with zero entries elsewhere, and in this case the third index in the tensor is therefore simply a one-hot representation of the integer. Finally the tensor representation for a list of length $ 1\leq i \leq L$ with integer elements $ k_j $, $ j = 0, 1, \ldots, i-1 $, has nonzero elements for $ \psi_{i+1, j, k_j - \ell_\textrm{min}} =1 $; in this case, again, the third index is a one-hot representation of $ k_j $.  We refer to these tensors as sharp states, as they uniquely correspond to a DSL variable. From the construction of the tensors, it is clear that the only elements that can take nonzero values are   $ \psi_{000}  $,  $ \psi_{10k} $ and $ \psi_{ijk}  $ for $ i \geq 2 $ and $0\leq j \leq i-2$. The remaining entries can therefore be regarded as a zero-padded region. 
 
Consider now a weighted sum $ \psi_{ijk} = \sum_{I=1}^r w_I\, \psi^{(I)}_{ijk}  $ of some sharp states $ \psi^{(I)}_{ijk} $, $ I = 1, \ldots, r $, where $ \sum_{I=1}^r w_I =1 $, so that $ w_I $ can be interpreted as the probability of the occurrence of state $ \psi^{(I)}_{ijk} $. In this case the tensor $ \psi_{ijk} $ is no longer a one-hot vector along its third dimension but a probability vector for different integers. The sum $ \sum_k\psi_{ijk}$, however, is not necessarily equal to 1 but gives the probability of having the variable type $ i $. This is regardless of the value of $j$ as far as it does not belong to the zero-padded region.
A further sum over the first index, however, must give
\be \label{psi=1}
\psi_{000} + \sum_k \psi_{10k} + \sum_{i\geq 2}\sum_k \psi_{ij_ik} =1 
\ee
for any $ 0\leq j_i \leq i-2 $, which simply means that any state is either null, an integer or a list of integers, whose probabilities are given respectively by the first, second and third terms on the left-hand side.

\section{Applying DSL functions}
\label{s:dsl}

\noindent
Having described the probabilistic $ \psi $ representation of the variables, we now turn to the transformation of these states under the DSL functions. We denote the transformed version of the state $ \psi_{ijk} $ by $ \tilde\psi_{ijk} $. Following \cite{balog2016deepcoder, zohar2018automatic}, the result of any function acting on an integer or the null state is defined to be null. We will therefore only need to specify their action on lists. Let's start our discussion of the DSL transformations with the simple function \texttt{head}. By definition, when acting on a list, the effect of the \texttt{head} function is an integer equal to the first  (leftmost) element of that list. Therefore, for the \texttt{head} operation, the only non zero elements of $\tilde \psi_{ijk}$ are $ \tilde \psi_{000} $ and $ \tilde \psi_{10k} $. The latter is related to the input state $ \psi $ as 
\be \label{head_10k}
\tilde \psi_{10k} = \sum_{i\geq 2} \psi_{i0k}
\ee
which means that the probability that the \texttt{head} operation gives an integer $ k+\ell_\textrm{min} $ is the probability of being in a list state with first element $ k+\ell_\textrm{min} $. This is regardless of the list length and hence the sum over $ i\geq 2 $. The other nonzero entry of the transformed state is given as $ \tilde \psi_{000} = \psi_{000} + \sum_k \psi_{10k} $. This is because the probability that the \texttt{head} operation gives the null state is equal to the probability of being either in a null state (hence the first term) or being in an integer state, independent of the integer value (hence the second term). In practice we will never need the $ \tilde \psi_{000} $ elements because the target outputs of the programs will never be null. Therefore from now on we can simply discard $ \tilde \psi_{000} $ and focus on the other nonzero elements of the transformed state. The transformation of the $ \psi_{ijk} $ state under the \texttt{tail} function, which picks instead the last entry of a list, is quite similar except that the r.h.s. of eq. (\ref{head_10k}) will be replaced by $ \sum_{i\geq 2} \psi_{i,i-2,k} $. 

Let us now move to the rest of the DSL functions. The remaining DSL functions we consider in this work are addition and subtraction by 1, multiplication and integer division by 2, 3 and 4, multiplication by $ -1 $, and raising to the power of 2. By definition, these operations apply to list elements only, and transform integers (that are not list elements) and the null states to null. For the integer divisions, the integer with the largest absolute value smaller than the absolute value of the normal division is taken. For instance $ 3 $ divided by $2$ will be $ 1 $ and $ -3 $ divided by $ 2 $ will be $ -1 $.

In principle the transformations of the states under these functions may involve nonlinearities because of the normalizations required to guarantee eq. \eqref{psi=1}. However, as we will see later, the precise form of the transformations that ensure their interpretation as probabilities is actually not necessary. In fact all we need is that the transformations are differentiable, are valid on sharp states, and that the transformed states satisfy $ \sum_k \tilde\psi_{10k} + \sum_{i\geq 2}\sum_k \tilde \psi_{ijk} \leq 1  $ (which is always valid if the state satisfies \eqref{psi=1}, given that we have discarded $ \tilde\psi_{000} $). This results from a theorem that we prove in the next section. These considerations allow for much simpler transformation rules. In particular under the above mentioned functions (apart from \texttt{head} and \texttt{tail}) the transformations can be written in a unified way as 
\be \label{transformations}
\tilde \psi_{ij\sigma(k)}= \psi_{ijk}.
\ee
where $ \sigma $ is a function of indices and depends on the DSL function. The range of $ k $ is such that $ 0 \leq \sigma(k) \leq d-1 $. The precise form of $ \sigma $ for the remaining DSL functions is listed below:  
{\setlength\arraycolsep{2pt}
\be  \label{functions}
\ba{llll}
\sigma(k)  &=& k + 1 &  \quad\texttt{plus1} \\
\sigma(k)  &=& k -1 &  \quad\texttt{minus1} \\
\sigma(k)  &=& 2\;k + \ell_\textrm{min} &  \quad\texttt{times2} \\
\sigma(k)  &=& 3\;k + 2 \; \ell_\textrm{min} &  \quad\texttt{times3} \\
\sigma(k)  &=& 4\;k + 3 \; \ell_\textrm{min} &  \quad\texttt{times4}
\ea \qquad \qquad
\ba{llll}
\sigma(k)  &=& -k -2 \; \ell_\textrm{min} &  \quad\texttt{timesm1} \\
\sigma(k)  &=& (k + \ell_\textrm{min})^2 - \ell_\textrm{min} &  \quad\texttt{power2} \\
\sigma(k)  &=& (k + \ell_\textrm{min}) / 2 - \ell_\textrm{min}   &  \quad\texttt{div2} \\
\sigma(k)  &=& (k + \ell_\textrm{min}) / 3 - \ell_\textrm{min} &  \quad\texttt{div3} \\
\sigma(k)  &=& (k + \ell_\textrm{min}) / 4 - \ell_\textrm{min}  &  \quad\texttt{div4}
\ea
\ee}
where the division in the last three equations is the previously defined integer division.

A property that the above functions have in common is that they result in a null output when they act on a null or integer input, and when they act on lists they either preserve the list length or give null.  By appropriately choosing the $ \sigma $ function, the transformation \eqref{transformations} applies to all other functions with the above-mentioned property. However, including other types of functions not enjoying this property requires working out the transformation $ \psi \rightarrow \tilde \psi $ separately. In this work, for simplicity of computation, we stick to the set of functions in eq. \eqref{functions} plus the \texttt{head} and \texttt{tail} operators.
Nevertheless, this choice of DSL is complex enough to challenge the existing program learning approaches, especially on long programs.

\section{Optimization process} 

The goal of the optimization process is to predict, at each timestep of the program, the probabilities of different functions. This prediction is based on the set of I/O examples in a sample, which are all consistent with a (not necessarily unique) program. For this purpose, the three dimensional tensors $ \psi $ for different input or output examples in a sample are collected together into a four dimensional tensor $ \Psi $ of shape $(m,L+2,L,d)$ whose first index $ m $ enumerates the  examples within a sample.

Let us now denote the DSL functions by $ f_s $, for instance these could be $ f_1= $ \texttt{head}, $ f_2= $ \texttt{tail}, etc. defined in the previous section. Then, the probability of the function $ f_s $ acting at timestep $ t $ of the program is denoted by $ \pi_{ts} $. Therefore the probability of all possible operations can be collected into a two dimensional tensor $ \pi_{ts} $ of shape $ (T, n) $, where $ T $ is the program length and $ n $ is the number of DSL functions. These probabilities clearly satisfy $\sum_s \pi_{ts} = 1$.

Let us denote the output state at timestep $t$ by $ \Psi_{(t)} $ which is an instance of the four-dimensional tensor $ \Psi $. This is related to the input from the previous timestep through the following weighted sum, or {\it superposition of states}
\be 
\Psi_{(t)} = \sum_s\pi_{ts}\,f_s(\Psi_{(t-1)} ).
\ee
Starting from the initial state $ \Psi_{(0)}  $ and by successive application of this operator, one arrives at the final state $  \Psi_{(T)}  $
\be 
\Psi _{in} \equiv \Psi_{(0)}  \rightarrow \Psi_{(1)}  \rightarrow \cdots \rightarrow \Psi_{(T)}  \equiv \Psi _{out},
\ee
therefore, for a fixed set of DSL functions, the tensor $ \pi_{ts} $ induces a map $ \pi_{ts}: \Psi _{in} \rightarrow \Psi _{out} $ between the input and output states. The output $ \Psi _{out} $ has then to be contrasted with the ground-truth output $ \hat \Psi $. For this purpose we need to come up with a suitable loss function.

Inspired by the cross-entropy loss the natural choice for this function would be
\be \label{loss}
L (\pi) =  -\frac{1}{N} \!\! \sum_{i\neq 0,m,j,k}\! \hat{\Psi}_{mijk} \,\log \, \Psi_{mijk} \textrm{,}  \quad \textrm{with} \quad
N = \!\!\!\sum_{i\neq 0,m,j,k}\! \hat{\Psi}_{mijk} \textrm{,}
\ee
where the argument of the $ \log $, $ \Psi = \Psi _{out} $ is the final output state of the program whose entries are the probabilities; $ \hat{\Psi} $ is the ground-truth tensor; and $ N $ is the total number of tokens in the ground-truth output variables. For each example $ m $, $ \hat{\Psi}_{mijk} $ is nonzero only for a single value of the index $ i $ but can have several nonzero entries in the remaining two dimensions $ j,k $ depending on how many nonzero entries there are in the given list. For a fixed program length the loss function, eq.(\ref{loss}), is then optimized with respect to the probabilities in $ \pi_{ts} $ to find $ \pi^*_{ts} = \mathrm{argmin}_{\pi} \; L(\pi) $. In practice, however, one may get stuck in local minima. For each timestep $ t $ of the program the $\mathrm{argmax}$ of the tensor $ \pi_{ts} $ along the second dimension $ s $ determines the predicted operation at that particular timestep. These operations are then put together to predict the final program.

In the previous section we anticipated that the transformations of the $ \psi $-representations under DSL functions only need to satisfy some weak conditions. Here we prove this explicitly by showing that under such conditions the solution $ \pi^*_{ts} $ to $ L(\pi)=0 $, when $\mathrm{argmax}$-ed over the second dimension, corresponds to the sequence of operations, i.e. the program, that maps the input $ \Psi_{in} $ to the ground-truth output $ \hat\Psi $. This is summarized in the following theorem:
\begin{theorem}
Under the following two conditions on the DSL transformations in the $ \psi $-representation: 
\begin{itemize}
	\item[(a)] $ \tilde \psi_{ijk} $ is a non-null sharp state if and only if $ \psi_{ijk} $ is a non-null sharp state;
	\item[(b)] If $ \sum_k \psi_{10k} + \sum_{i\geq 2}\sum_k \psi_{ij_ik} \leq 1 $ then $ \sum_k \tilde \psi_{10k} + \sum_{i\geq 2}\sum_k \tilde \psi_{ij_ik} \leq 1 $,  ($ 0\leq j_i \leq i-2 $);
\end{itemize}
a solution $ \pi^*_{ij} $ to equation $ L(\pi)=0 $ satisfies:
\be \label{f}
f_{\mathrm{argmax}_s \pi^*_{Ts}}\circ \cdots \circ f_{\mathrm{argmax}_s \pi^*_{2s}}\circ f_{\mathrm{argmax}_s \pi^*_{1s}} (\Psi_{in}) = \hat\Psi.
\ee
\end{theorem}
\begin{proof}
The equation $ L(\pi^*)=0 $ implies that the predicted output $ \Psi_{(T)}  = \Psi_{out}$ will have 1's at the same position as those of $ \hat \Psi $. Then condition (b) guarantees the vanishing of the remaining entries of $ \Psi_{(T)} $. Therefore $ \Psi_{(T)} $ will be a sharp state equal to $ \hat \Psi $. Given that $ \Psi_{(T)}  = \sum_s \pi_{Ts} f_s (\Psi_{(T-1)}) $, where $\sum_s \pi_{Ts} = 1$, and that the entries in $ f_s (\Psi_{(T-1)}) $ are all less than 1 (which follows from condition (b)), we can conclude that all $ f_s (\Psi_{(T-1)}) $ with nonzero coefficients $ \pi_{Ts} $ have the same structure as $ \hat \Psi = \Psi_{(T)} $, in particular $ f_{\mathrm{argmax}_s\pi^*_{Ts}}(\Psi_{(T-1)}) = \Psi_{(T)}  $. Now, from (a) $ \Psi_{(T-1)} $ will also be a sharp state, and therefore the argument can be repeated. This can be continued to conclude that $ f_{\mathrm{argmax}_s\pi^*_{1s}}(\Psi_{(0)} = \Psi_{in}) = \Psi_{(1)} $, so that $ f_{\mathrm{argmax}_s\pi^*_{ts}}(\Psi_{(t-1)}) = \Psi_{(t)}  $ for $ t=1, \cdots, T $, which implies eq. \eqref{f}.
\end{proof}
It is straightforward to show that our functions defined in the previous section satisfy these two requirements. This is explicitly demonstrated in the appendix.  

The tensor $ \psi $ not only gives a probabilistic description of possible states but it is highly compressed, in the sense that its number of elements $d \; L \; (L+2)$ is polynomial in $ d, L $, and independent of the program length. This should be compared with the one-hot vector representation of the DSL variables. The length of such a vector will be roughly the number of possible lists which is exponentially large, of the order $ d^L $. 
The trade-off for this compact choice of representation is that information is lost in superpositions of states, that is, if we sum up two states, the individual summands cannot necessarily be completely recovered, contrary to the one-hot representation where each summand contributes to a different vector component and thus no information is lost. Nevertheless, despite this loss of information, our choice of representation still captures enough information to allow identifying the program that maps input to output, providing the ground for Theorem 1.

\section{Neural network}
\label{s:net} 

In the optimization process the variables $ \pi_{ts} $ may be chosen to be initiated randomly. However, a more educated choice leads to a better performance. This is where a neural network can prove highly effective. 
Specifically, one may train a neural network to predict the program consistent with a given set of I/O examples, that is, the input of the network is a set of I/O pairs consistent with some program, and the output is a two dimensional tensor of probabilities of the same shape as $ \pi_{ts} $. This tensor is then fed as an initialization point into the optimization problem described in the previous section. This way, the neural network integrates with and guides our optimization method.

The natural choice of a neural network for this purpose would be a sequence-to-sequence model. We have implemented in PyTorch \cite{NEURIPS2019_9015} an attention-based recurrent neural network (RNN) model with a bidirectional encoder, a unidirectional decoder, and single-layer gated recurrent units (GRU) \cite{cho2014learning} both in the encoder and decoder. The encoder and decoder are connected only through the attention network. Further details about the network, including the architecture and its performance are provided in the appendix.

\section{Experiments and results}

In this section we present the results of two different experiments regarding program learning, and compare our SuperCoder approach with a state-of-the-art method, PCCoder \cite{zohar2018automatic}.

We have prepared a dataset of 86604 sample programs which is split with a 9 to 1 ratio into training and validation datasets. 
Each sample includes a set of I/O examples that are all consistent with a specific program.
We set the maximum length of the programs with which the samples are generated to 25. The set of DSL functions are those defined in Section \ref{s:dsl}, and the integers and list elements range from $\ell_\textrm{min}=-100$ to $\ell_\textrm{max}=100$. The lists have a maximum length of $L = 10$. 

We further have a set of 18 separate test datasets, each including 500 samples and generated with programs having lengths between 8 and 25. This allows us to test our model on problems of different program lengths independently. In the test datasets, the number of I/O examples in each sample depends on the task, as will be further detailed below. The training and test datasets are generated using the open source implementation of PCCoder\footnote{https://github.com/amitz25/PCCoder.}.

It is worth mentioning that, given our set of ten operations defined in eq.\eqref{functions} (putting aside the \texttt{head} and \texttt{tail} operators for simplicity), for the shortest program length we consider in this work (i.e. 8) the number of possible programs 
would roughly be $10^8$, which grows exponentially with program length up to $10^{25}$ for the longest programs. The search space is therefore large enough to pose a challenge for traditional symbolic search algorithms. 

Specifically, for the problem considered in this work, it takes around $10^{-3}$ seconds to generate a random program and verify its compatibility with a given sample. This means that in one second a thousand random programs can be tested, which is too small compared to the size of the search space, even for the shortest programs considered here, and taking into account possible degeneracies in program space. Indeed, we have explicitly checked, for the test dataset generated with length-8 programs, that allocating 5 seconds for each sample (similar to the experiments of the following section), random search is able to solve only one out of 500 samples. 

In contrast, SuperCoder's search method, which is in fact an optimization process, is almost agnostic to the program length. As demonstrated by the experiments of the following sections, this low sensitivity to program length can be seen 
even in comparison to PCCoder as a state-of-the-art search method that is guided by a network and is especially designed to synthesize long programs.

\subsection{Synthesizing programs that reproduce exactly the I/O mappings}

For this setup, in the test datasets we will have 5 examples in each sample. We consider that the synthesized programs are correct if they are able to reproduce exactly the I/O mappings for each of the 5 examples. The accuracy we measure is the percentage of synthesized programs that are correct. This is done separately for test datasets generated by programs of different lengths.

While we know the length of the program that generated a sample, it is possible that shorter programs may generate the sample as well. Computing the length of the shortest program that performs a given I/O mapping (conditional Kolmogorov complexity) is known to be uncomputable. 
Nevertheless, this is not our focus in the present work, rather, the goal here is to see if program synthesis methods are able to find, within a given timeout, a program of prespecified length 
consistent with I/O examples in the sample. 
One such program is the one that generated the sample in the first place.

For the optimization process, SuperCoder takes as an input the given length of the program that generated the sample. To make a fair comparison, we therefore restrict the search algorithm of PCCoder as well 
so that it dedicates all the timeout to searching for programs of the given length.

The synthesis accuracies on samples generated by programs with lengths between 10 and 17 are reported in Table \ref{synthesis_comparison}, where the timeout for search/optimization is set to 5 seconds. The table shows that SuperCoder performs better than PCCoder for program lengths longer that 10.

\begin{table}[h]
\begin{center}
\begin{tabular}{rccccccccccc}
\Xhline{1pt} 
	\it{Program length} & 10 & 11 & 12 & 13 & 14 & 15 & 16 & 17    \\ \hline
	PCCoder & {\bf 72.4\%} & 49.2\% & 31.4\% & 42.2\% & 34.2\% & 52.8\% & 29.4 \% & 40.6\%   \\ 
	SuperCoder & 53.4\% & {\bf 63.2\%} & {\bf 54.6\%} & {\bf 44\%} & {\bf 36.8\%} & {\bf 72.2\%} & {\bf 60\%} & {\bf 47.8\%} \\ \Xhline{1pt} 
\end{tabular}
\end{center}
\caption{SuperCoder vs PCCoder. Comparison of synthesis accuracies as a function of different program lengths. SuperCoder achieves state-of-the-art records on long programs.}
\label{synthesis_comparison}
\end{table}

\subsection{Assessing synthesized programs for approximate prediction of outputs}

For this task our test datasets include 10 examples in each sample, which are split 
into 5 observed and 5 assessment examples. The program is then inferred from the 5 observed examples and tested on the 5 assessment examples.

In this case the performance of the synthesized program is measured by computing a score that measures the similarity between the program outputs and the target outputs for the assessment examples. This is a ratio of two integers calculated as follows: For each assessment example the model gets a credit if the type of the output, i.e. whether it is a list or an integer, is predicted correctly, and then an extra credit for each token of the output that it correctly predicts at each position in the output. The results are summed for all assessment examples in the sample and form the numerator. The denominator instead receives a contribution from each example, which is equal to 1 plus the maximum of the lengths of the prediction and the ground-truth output. The 1 here corresponds to the {\it type} token. The length for the integer type variables is 1.

We introduce in the dataset various amounts of noise in the outputs (and not in the inputs). This is done by replacing each output token with a random integer within the allowed range, with a certain probability, i.e. the noise. For the training/validation datasets, the noise is introduced in all the outputs, while for the test datasets this is done only for the outputs of the observed examples. Specifically, we take the training/validation datasets as well as the test datasets of different program lengths, and introduce 10\%, 20\% and 30\% noise into them. We will therefore have four training datasets, one without noise and 3 with different amounts of noise. Under each of these four datasets we train our attention-based RNN over 50 epochs and pick the model with lowest validation error, rather than the one of the last epoch. The training is done using Adam optimizer \cite{kingma2014adam} with a learning rate of 0.0005 and batch size 32. The resulting four trained models are then further improved with the optimization process and tested separately on the test datasets with the same amount of noise, for the 18 different program lengths. For the test, the model predicts the program from observing the first 5 examples in a sample (i.e. observed examples), and applies the predicted program to the inputs of the assessment examples, which are finally contrasted with the outputs of the assessment examples, in order to compute the accuracy. For each sample we run the optimization process for 5 seconds on a GeForce GTX 1080 Ti GPU, with a learning rate of 0.2. 

 We compare the performance of our method with PCCoder. The network of PCCoder is trained over 40 epochs under the same four datasets, and the model with the lowest validation error is chosen. The validation errors usually become stable within the first 15 epochs of training. The trained models with different amounts of noise are then combined with the search process and tested separately on the test datasets with the same amount of noise, for the 18 different program lengths. 
\begin{figure}[h]
\centering
\includegraphics[width=0.45\textwidth]{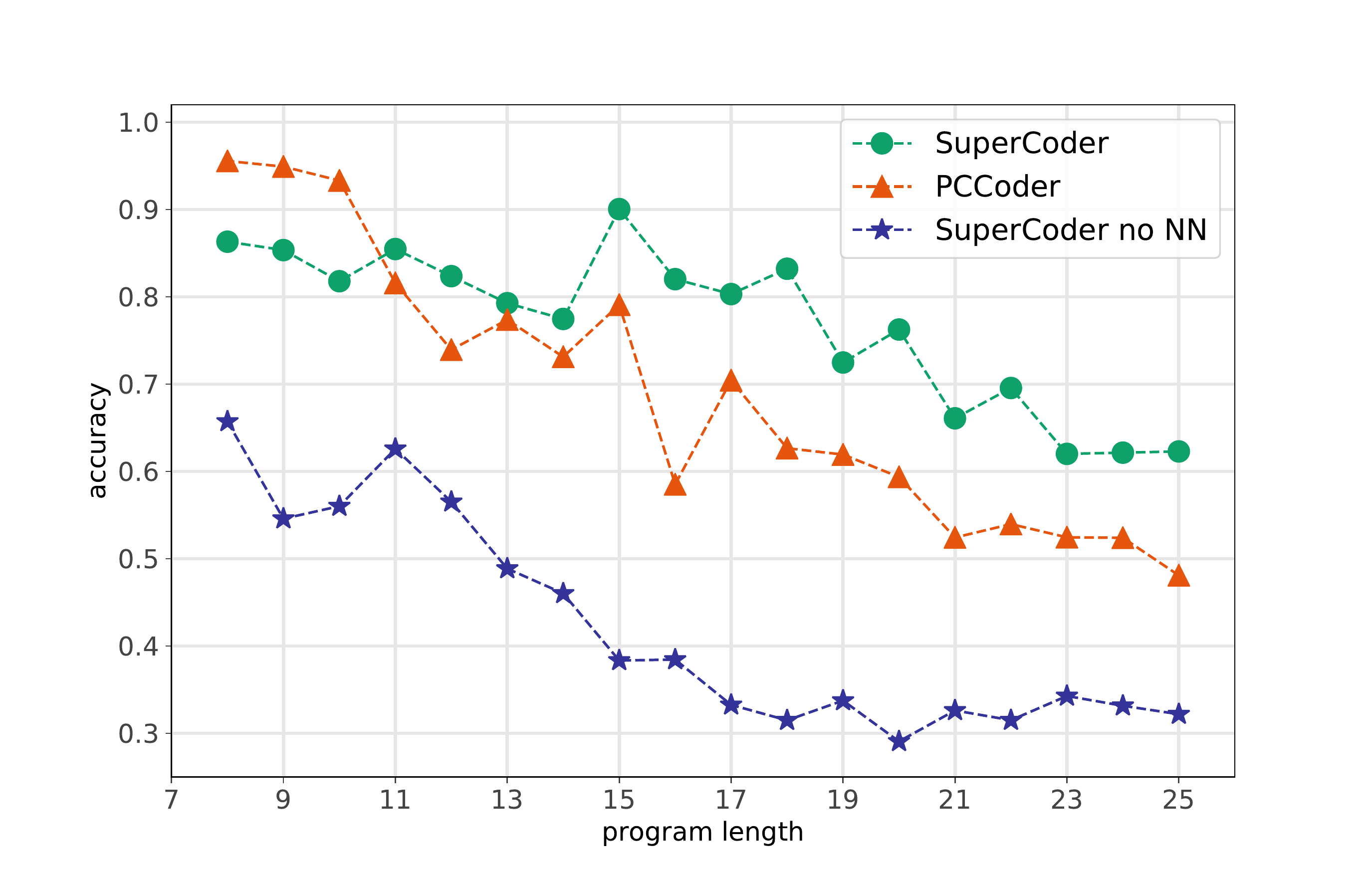}
\includegraphics[width=0.45\textwidth]{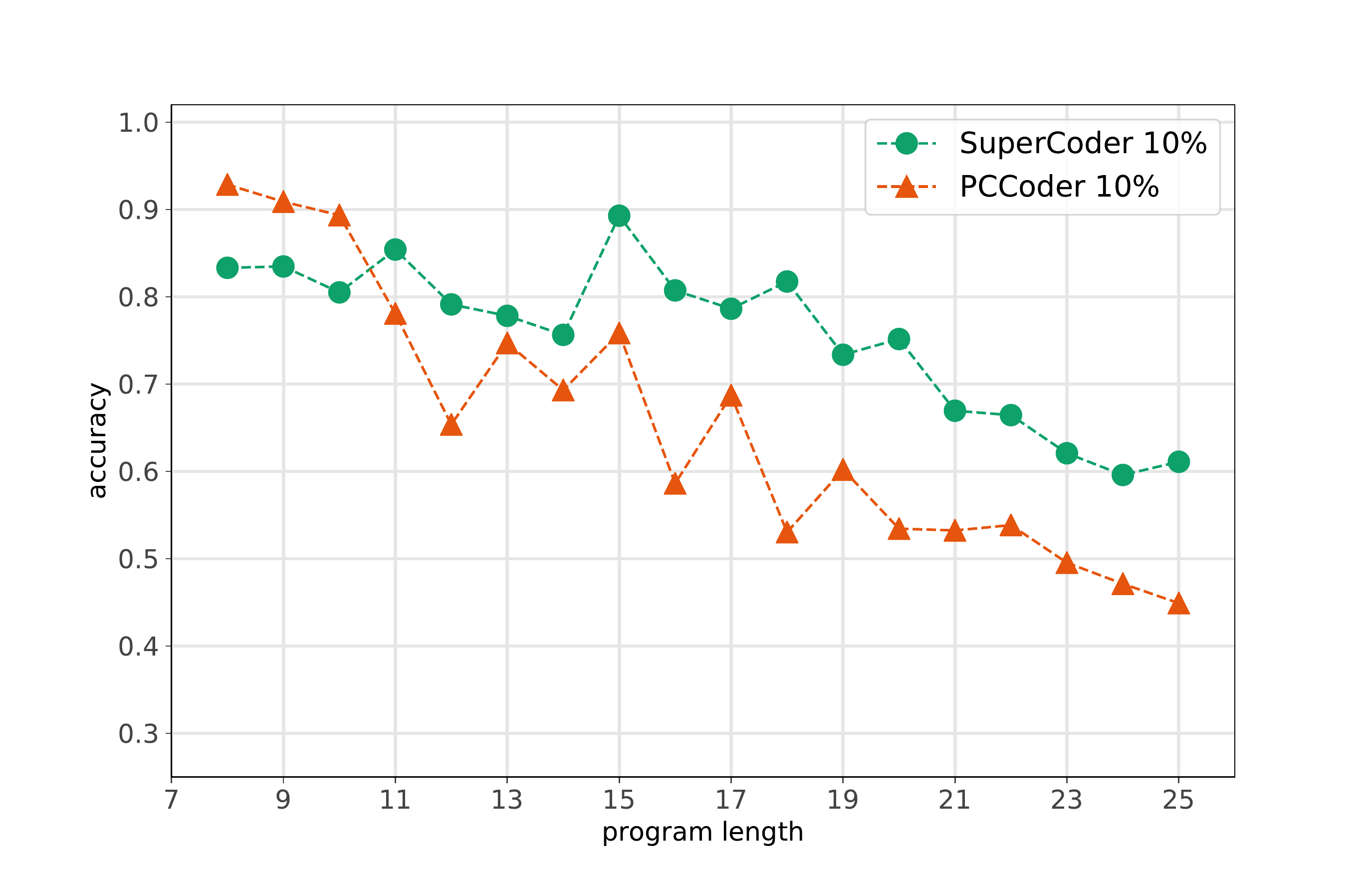}
\includegraphics[width=0.45\textwidth]{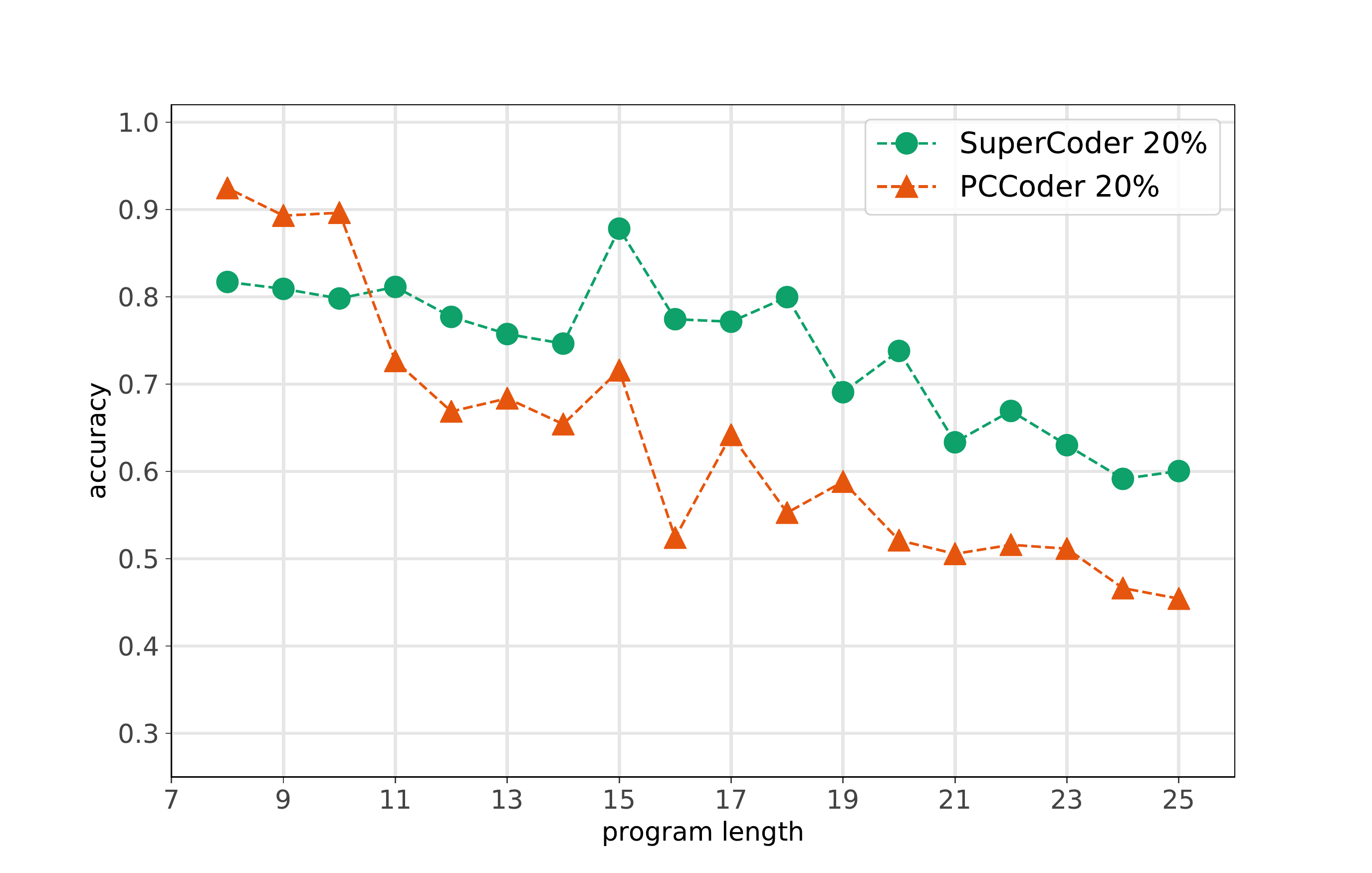}
\includegraphics[width=0.45\textwidth]{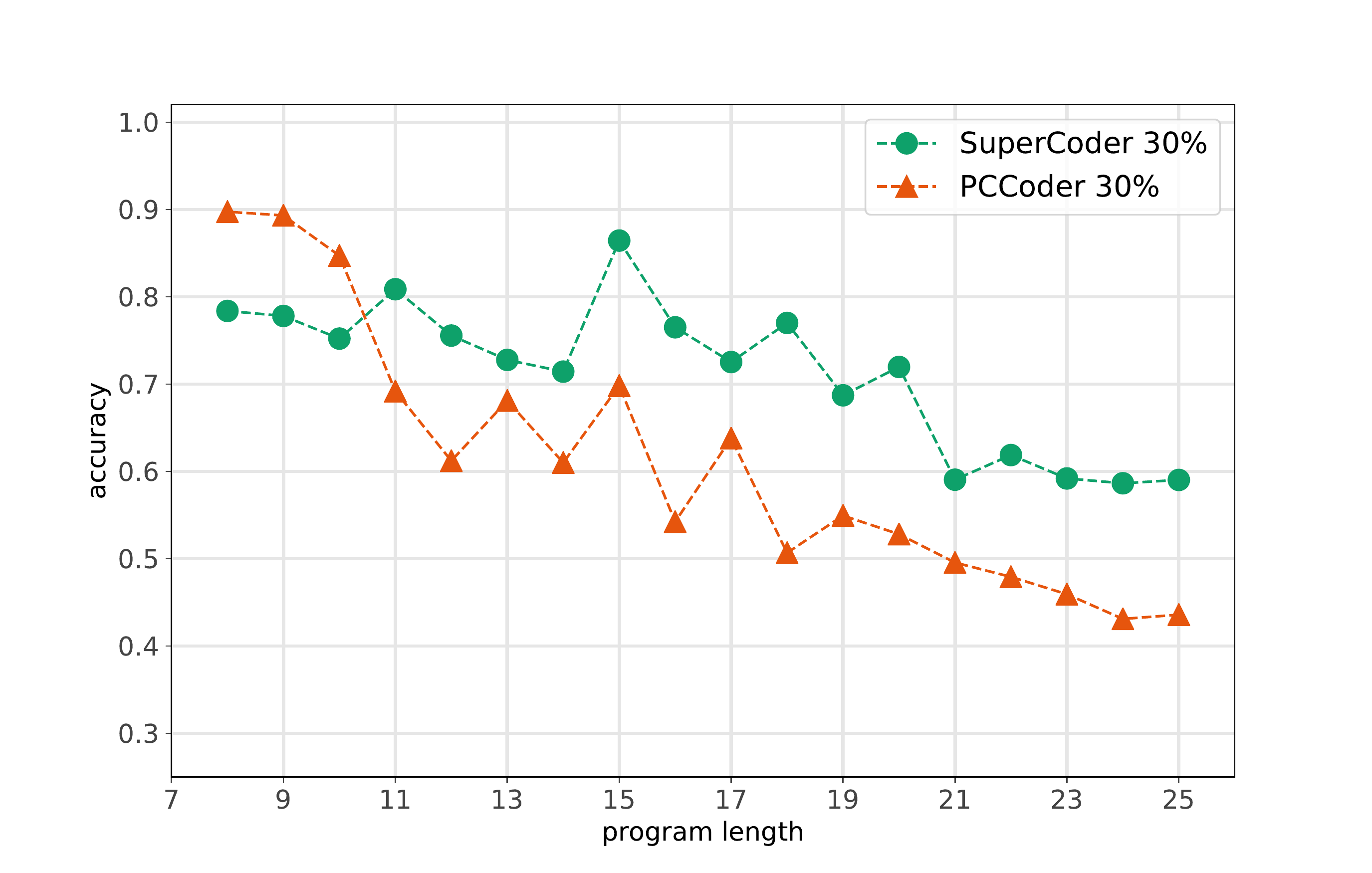}
\caption{Comparison of output prediction accuracies for SuperCoder and PCCoder, 
on a clean dataset (top left), as well as datasets with 10\%, 20\% and 30\% noise. SuperCoder outperforms PCCoder on long programs in all cases. On the top left plot we have included the prediction of SuperCoder with no network guidance, and random initialization.}
\label{induction_accuracy}
\end{figure}
The default PCCoder simply does not work on noisy datasets, however we can adapt it to noise with a slight modification. PCCoder uses a beam search to identify the first ``beam size'' most likely programs, and checks whether any of them is the correct program, mapping input to output. This is not meaningful when dealing with noisy data because a ``correct'' program does not necessarily exist. 

In the original PCCoder, the beam search is performed in a spiral manner, starting from a beam size of 100, and doubling the size if the program is not found, until a certain timeout is reached. Here, the notion of a consistent program does not apply, and instead we look for a program that gives the ``best'' output, i.e. one that shares the most tokens with the ground-truth output. For this purpose, instead of checking the program in the beam for consistency with input/output we calculate the score we defined earlier for all the programs in the beam, and choose the one with the highest score. To compare with SuperCoder, we allow the same timeout of 5 seconds for the search process. Given that there is no particular time at which the search finds a desired program and stops, and that a fixed timeout is specified, applying a spiral beam search is not the best option, rather, in order to take the most advantage of the whole allocated timeout, the optimum method would be to fix the beam size and set it to the highest value the algorithm can finish within the timeout. The beam size corresponding to the 5 seconds timeout depends on the program length. For each test dataset with a fixed program length, we scan over all beam sizes that are multiples of 10, identify the largest beam size that the algorithm finishes in 5 seconds, and choose that beam size plus 10 (in other words, we choose the smallest beam size that the algorithm cannot finish in 5 seconds). This search part is done on the observed examples in the sample, after which the chosen program will be applied to the inputs in assessment examples for the calculation of the token accuracy. We believe this is the optimum modification to PCCoder with which one can get the best performance on noisy data, and it indeed is highly efficient. 
The accuracies of SuperCoder and PCCoder for different amounts of noise and for different program lengths are shown in Fig.\ref{induction_accuracy}. To clarify the complementary roles of the network and the optimization process, we include, for the clean dataset, the result of the same experiment but with no network assistance, and random initialization instead. This shows that the network has a significant role for improving the optimization process of SuperCoder.

\section{Conclusions}

We have addressed the problem of automatic program learning in a domain specific language. We have developed a novel method, SuperCoder, inspired from quantum superposition in physics, where inferring a program or an output is done by optimizing a cost function through gradient descent. The performance of this process is improved by using a distinct neural network that provides initial points for the other optimization. SuperCoder improves the state-of-the-art for synthesising longer programs that use a specific DSL focused on operations on lists of integers. 
On clean datasets of I/O, corresponding to long programs, our SuperCoder approach can synthesize a higher number of programs compared to PCCoder. We also test the two approaches on noisy datasets. For this, we introduce a slight modification of PCCoder's search method which enables it to predict an output directly, hence turning it into a program learning method able to deal with noisy data. Experiments on test datasets of different program lengths and different amounts of noise show a significant increase in token accuracies on long programs with respect to the modified PCCoder. 

These improvements are the result of an interplay between several features of our method. First, contrary to typical search methods, where the number of variables increases at each step of the program, our probabilistic representation of variables captures all possibilities in a single three-dimensional fuzzy tensor whose structure remains the same through different steps of the program.  Moreover, there is no direct search involved and instead the likely program is inferred through an optimization process.

\section*{Acknowledgment}
This work was supported by the European Regional Development Fund and the Romanian Government
through the Competitiveness Operational Programme 2014-2020, project ID P\_37\_679, MySMIS code
103319, contract no. 157/16.12.2016.

\bibliography{supercoder_arxiv}
\bibliographystyle{abbrv}

\appendix

\section{Requirements of Theorem 1}

It is straightforward to show that DSL transformations of the form \eqref{transformations}, corresponding to the $ \sigma $ functions defined in \eqref{functions} satisfy requirements (a) and (b) of Theorem 1. Condition (b) holds because
\be 
\sum_k \tilde \psi_{10k} + \sum_{i\geq 2}\sum_k \tilde \psi_{ijk} = 
\sum_{i\geq 2}\sum_k \tilde \psi_{ijk} = 
\sum_{i\geq 2}\sum_{k_\sigma} \tilde \psi_{ij\sigma(k)} = 
\sum_{i\geq 2}\sum_{k_\sigma} \psi_{ijk} \leq 
\sum_{i\geq 2}\sum_k \psi_{ijk} \leq 1
\ee
where $ k_\sigma = \{k| 0 \leq \sigma(k) \leq d-1\} $, and in the first equation we have used the fact that $ \tilde \psi_{10k} =0 $ for such transformations. Also, the transformations \eqref{transformations} have been defined such that they map sharp states to sharp states. On the other hand, by eq. \eqref{transformations} if $ \tilde \psi_{ijk} $ is a non-null sharp state, $ \tilde \psi_{ijk} $ will differ from $ \tilde \psi_{ijk} $ only by the position of its nonzero elements along the third dimension, so it will be a non-null sharp state as well. Similar arguments show that these requirements are valid also for the \texttt{head} and \texttt{tail}  functions. The \texttt{head} function satisfies
\be 
\sum_k \tilde \psi_{10k} + \sum_{i\geq 2}\sum_k \tilde \psi_{ijk} = 
\sum_k \tilde \psi_{10k} = 
\sum_{i\geq 2}\sum_k \psi_{i0k}  \leq 1
\ee
which shows that it fulfils our condition (b). For the \texttt{tail} function, one only needs to replace $ \psi_{i0k} $ with $ \psi_{i,i-2,k} $. The \texttt{head} and \texttt{tail}  functions also obviously map sharp states to sharp states. However, a non-null sharp state $ \tilde \psi_{ijk} $  in this case does not in general imply a non-null sharp state $ \psi_{ijk} $. Nevertheless, a non-null sharp $ \psi_{ijk} $ is still guaranteed in our particular case where the DSL functions never mix lists of different length.

\section{Neural network details}

In Section \eqref{s:net} we  described briefly the neural network used to assist our optimization process. Here we provide some further details. The network architecture is shown in Fig. \ref{net}.
\begin{figure}[h]
	\centering
	\includegraphics[width=0.5\textwidth]{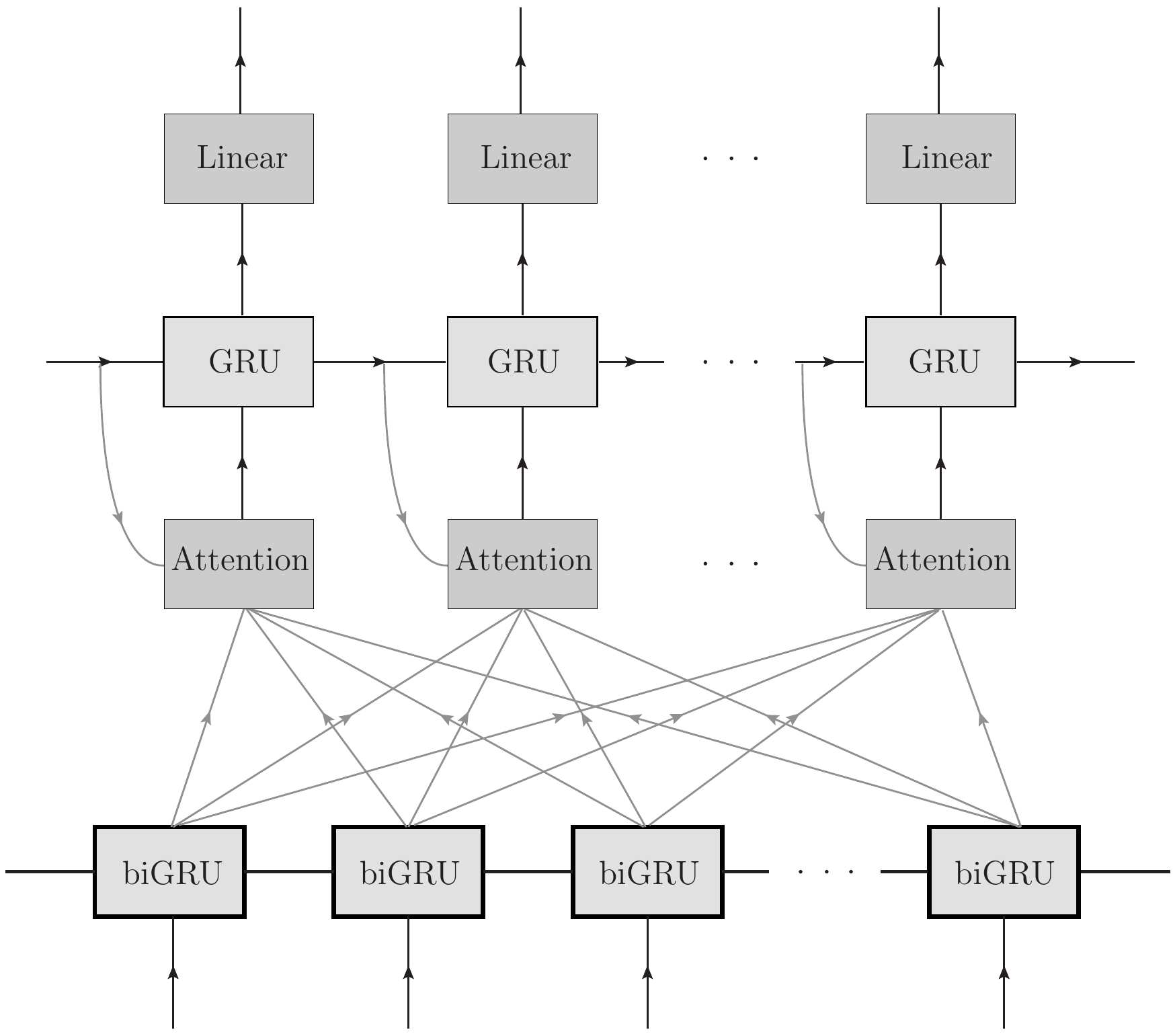} 
	\caption{Network architecture.}
	\label{net}
\end{figure}
The network and the training process do not rely on the representation discussed in Section \eqref{s:dsl}. Instead the variables are represented by a list of length $ L+2 $, where the first two entries take either $ 1, 0 $ or $ 0, 1 $, and specify whether the variable is an integer or a list. The remaining entries starting from the third one are filled with the index (between $ 0 $ and $ d-1 $) corresponding to the integers in the variable and the rest are padded with $ d $.

As explained in main text, each sample in the dataset includes a number of I/O pairs. The I/O in each pair are concatenated to give a list of length $ 2(L+2) $. For each sample these are then fed into the network as a batch of (in our case 5) examples. In both encoder/decoder RNN's an embedding is applied on the inputs before entering the GRU's, to reduce dimensionality of the indices.

To have an idea of the performance of this network we have defined three types of accuracies. In the first type which we call ``token accuracy'' all steps of the predicted sequence of operations are considered separately and any correct prediction of a single step in the sequence will count. The second type ``token accuracy top $ k $'' is similar to the first one except that instead of a single prediction we take the top $ k $ most likely outcomes and a prediction is counted as correct if the ground-truth label is within this set of $ k $ elements. Finally, in the third type of accuracy ``sequence accuracy top $ k $'' if all the steps of a ground-truth sequence are within the top $ k $ most likely values (predictions) of that step the sequence is said to be predicted correctly and counts in the accuracy. In this analysis we have taken $ k=5 $. In Fig. \ref{net_acc} we report all three accuracies after each epoch of training, for a total of 50 epochs. Training on each epoch takes roughly 105 seconds on a GPU. We can see that the ``sequence accuracy top $ k $'' starts with a low value but a large slope, so that it finally converges asymptotically to the ``token accuracy'' after 50 epochs and reaches around 52\%.
\begin{figure}[h]
	\centering
	\includegraphics[width=0.6\textwidth]{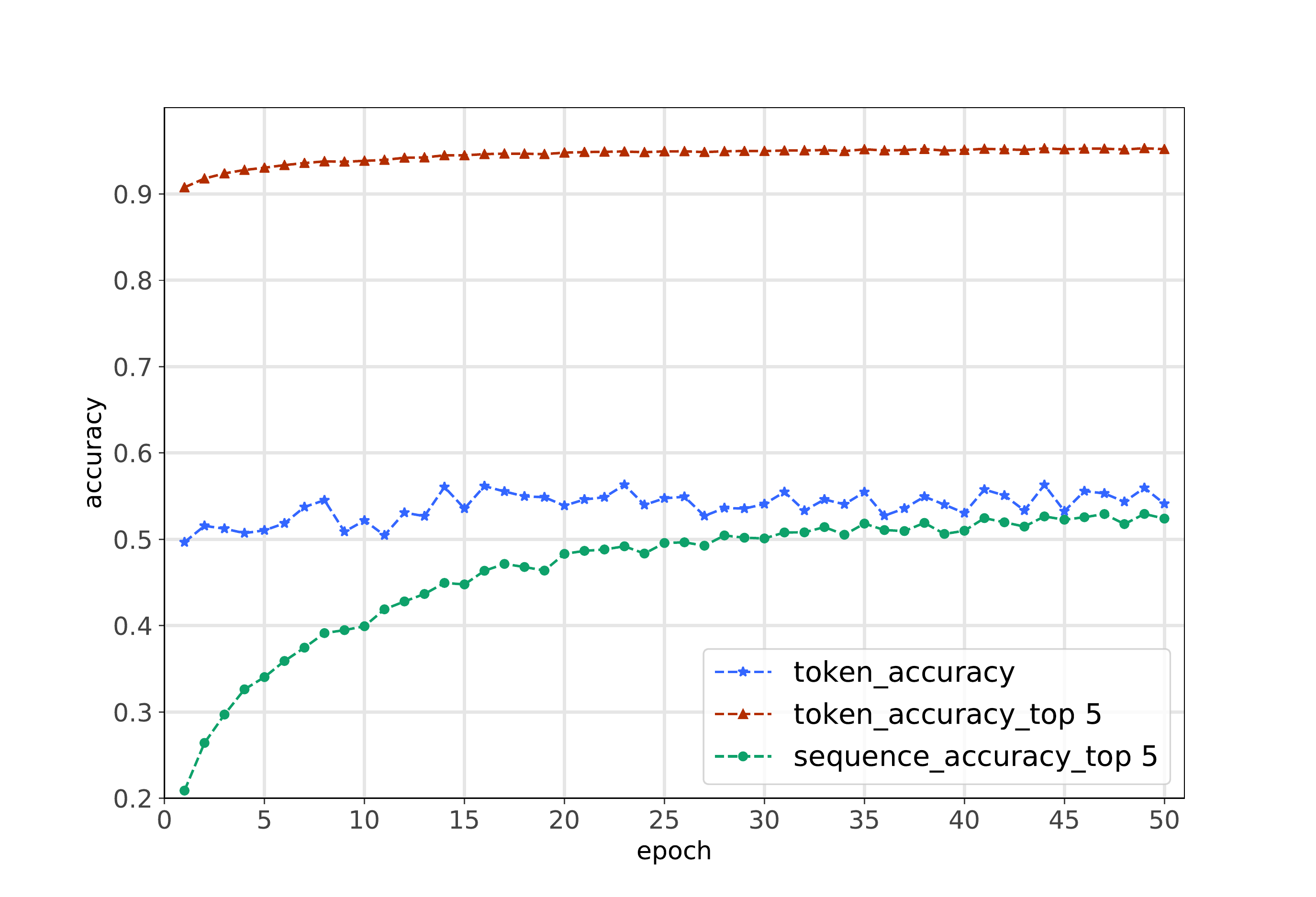}
	\caption{Token and sequence validation accuracies of the neural network predictions over training epochs. The percentages in the parenthesis are accuracies of the final epoch.}
	\label{net_acc}
\end{figure}

\end{document}